\documentclass[letterpaper]{article} % DO NOT CHANGE THIS
\usepackage{aaai25}  % DO NOT CHANGE THIS
\usepackage{times}  % DO NOT CHANGE THIS
\usepackage{helvet}  % DO NOT CHANGE THIS
\usepackage{courier}  % DO NOT CHANGE THIS
\usepackage[hyphens]{url}  % DO NOT CHANGE THIS
\usepackage{graphicx} % DO NOT CHANGE THIS
\usepackage{natbib}  % DO NOT CHANGE THIS AND DO NOT ADD ANY OPTIONS TO IT
\usepackage{caption} % DO NOT CHANGE THIS AND DO NOT ADD ANY OPTIONS TO IT
\usepackage{algorithm}
\usepackage{algorithmic}
\usepackage{ragged2e}
\usepackage{blindtext}

\usepackage{latexsym}
\usepackage{amssymb}
\usepackage{amsmath}
\usepackage{amsthm}
\usepackage{booktabs}
\usepackage{enumitem}
\usepackage{graphicx}
\usepackage{color}

\usepackage{times}
\usepackage{soul}
\usepackage{url}
\usepackage[utf8]{inputenc}
\usepackage{caption}
\usepackage{graphicx}
\usepackage{amsmath}
\usepackage{booktabs}
\usepackage[justification=centering]{caption}

\def\R{{\mathbb{R}}}

\def\E{{\mathbb E}}
\def\X{{\mathcal X}}

\def\A{{\mathcal A}}
\def\H{{\mathcal H}}

\def\D{{\mathcal D}}

\usepackage{algorithmic,algorithm}

\newtheorem{theorem}{Theorem}

\newcommand{\BibTeX}{B\kern-.05em{\sc i\kern-.025em b}\kern-.08em\TeX}

\DeclareMathOperator*{\argmax}{arg\,max}
\def\nys{Nystr{\"o}m }

\usepackage{newfloat}
\usepackage{listings}
\DeclareCaptionStyle{ruled}{labelfont=normalfont,labelsep=colon,strut=off} % DO NOT CHANGE THIS
\floatstyle{ruled}
\newfloat{listing}{tb}{lst}{}
\floatname{listing}{Listing}
\title{Bridging the Gap Between Hyperdimensional Computing and Kernel Methods via the \nys Method}
\author{
    Quanling Zhao,
    Anthony Hitchcock Thomas,
    Ari Brin,
    Xiaofan Yu,
    Tajana Rosing
}
\affiliations{
Department of Computer Science and Engineering, UC San Diego\\
\{quzhao,ahthomas,abrin,x1yu,tajana\}@ucsd.edu
}

\usepackage{bibentry}
\begin{document}

\maketitle

\begin{abstract}
Hyperdimensional computing (HDC) is an approach from the cognitive science literature for solving information processing tasks using data represented as high-dimensional random vectors. The technique has a rigorous mathematical backing, and is easy to implement in energy-efficient and highly parallel hardware like FPGAs and ``processing-in-memory'' architectures. The effectiveness of HDC in machine learning largely depends on how raw data is mapped to high-dimensional space. In this work, we propose NysHD, a new method for constructing this mapping that is based on the \nys method from the literature on kernel approximation. Our approach provides a simple recipe to turn any user-defined positive-semidefinite similarity function into an equivalent mapping in HDC. There is a vast literature on the design of such functions for learning problems. Our approach provides a mechanism to import them into the HDC setting, expanding the types of problems that can be tackled using HDC. Empirical evaluation against existing HDC encoding methods shows that NysHD can achieve, on average, 11\% and 17\% better classification accuracy on graph and string datasets respectively.

\end{abstract}


 \begin{links}
     \link{Code}{https://github.com/QuanlingZhao/NysHD}
 \end{links}

\section{Introduction}
\label{intro}
Biological brains “compute” using data representations that are intrinsically fault-tolerant, suitable for highly-parallel circuitry, and reveal complex structures in an environment that are easy to learn \cite{hertz2018introduction}. Motivated by these desirable qualities, hyperdimensional computing (HDC) builds on theories of representation from cognitive science \cite{kanerva2009hyperdimensional,plate1995holographic} to develop novel hardware and algorithms for information processing tasks.
In HDC, all computation is performed using high-dimensional, low-precision, vector representations of data. These representations can be manipulated using simple, element-wise operators, so as to implement learning algorithms or other information processing tasks.

In contrast to deep learning models, training HDC-based models can typically be done in a single pass over the training data \cite{hernandez2021onlinehd,yu2022understanding} and  does not require back-propagation. The operations used in HDC are lightweight and highly parallelizable, making them suitable for implementation on low-energy and parallel hardware platforms.  This makes HDC an attractive alternative for implementing learning in resource constrained settings. As a result, HDC has gained significant interest in recent years, especially in Internet of Things (IoT)~\cite{khaleghi2022generic,zhao2022fedhd,morris2021hydrea}; and in the computer hardware community~\cite{dutta2022hdnn,kang2022xcelhd} such as FPGAs~\cite{salamat2019f5}, GPUs~\cite{kang2022openhd}, ASICs~\cite{zhang2023accelerating} and in-memory computing~\cite{dutta2022hdnn,xu2023fsl}.

The first stage in any HDC task is encoding, which maps data from its ambient representation $x \in \mathcal{X}$, into a representation $\phi(x)$ residing in a high-dimensional inner product space $\mathcal{H}$ \cite{kleyko2023survey}. HDC uses addition and multiplication (called `bundling'' and ``binding'' in the HDC literature), to build representations of complex structures from simple building blocks or to implement tasks like learning. For instance, classification using HDC represents each class as a sum of the encodings of its training data, called a ``prototype.'' Inference can then be performed by finding the closest prototype to a query.

The crucial assumption underlying the success of this technique is that ``similar'' points in $\mathcal{X}$ are mapped to ``similar'' regions of $\mathcal{H}$. In practice, this desideratum typically means that dot-products in $\mathcal{H}$ should be reflective of some salient notion of similarity on $\mathcal{X}$. In HDC, one typically builds representations incrementally, by bundling and binding together the embeddings of simpler atoms. In this work, we observe that one can also go in the opposite direction: starting from a known similarity function of interest, it is possible to generate an equivalent---up to some approximation factor---encoding function.

The advantage of this ``top down'' approach is that there is a vast literature on designing good similarity functions for different kinds of learning problems. In machine learning, similarity functions that work by computing inner products between high-dimensional embeddings of data are called kernel functions, and are the basis of kernel methods, a vast area of research in theoretical and applied ML \cite{shawe2004kernel,smola1998learning}. This literature has devoted substantial attention to the problem of designing good similarity functions, which are also potentially applicable to the kinds of problems encountered in HDC. In this work, we study an approach, based on the \nys method from the literature on kernel approximation \cite{williams2000using}, which can take a user defined kernel and generate an low-precision, randomized embedding, suitable for use in the kinds of learning algorithms employed in HDC. In a nutshell, the contributions of this paper are as follows:
\begin{itemize}
    \item We propose NysHD, a new way to generate embeddings for HDC which can turn any user-defined kernel function into an equivalent encodings.
    \item We analyze the kernel-preserving properties of NysHD formally and show that the inner product between encoded samples preserves normalized kernel values.
    \item We perform an empirical evaluation against existing HDC encoding methods and various neural network architectures and show our method substantially improves the performance of HDC-based learning, while preserving efficiency benefits relative to DNNs.
    \end{itemize}
From a practical standpoint, our work has the potential to expand the scope of tasks that can be effectively addressed using HDC by allowing practitioners to access the large repertoire of kernels that have been designed for them (i.e. graph kernels, string kernels etc.).

\section{Background and Related Work}

\subsection{Learning With HDC}
\label{2.1}

\begin{figure}[t]
  \centering
  \captionsetup{justification=centering}
\includegraphics[width=0.5\textwidth]{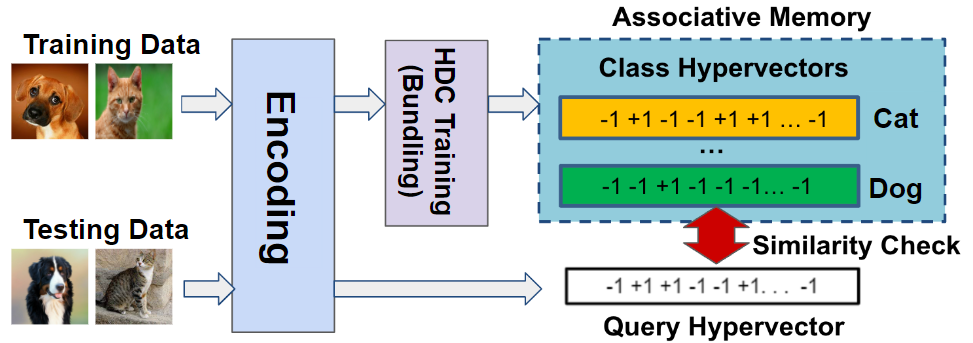}
  \caption{Overview of HDC training and inference for classification tasks.}
  \label{fig:hd_overview}
\end{figure}

In this work, we focus on using HDC to solve classification problems, which is a common practical application of the technique \cite{imani2019bric,rahimi2018efficient,menon2022highly}.
Fig~\ref{fig:hd_overview} demonstrates a typical HDC learning workflow for classification~\cite{morris2021hydrea,nunes2022graphhd,miranda2022hyperdimensional}. Let $\D = \{(x_{1},y_{1}),...,(x_{n},y_{n})\}$ be a set of training data, where $x_{i} \in \X$ is an input, and $y_{i} \in \{1,...,c\}$ is a class label. The first step is to embed the training data into a $d$-dimensional inner product space $\H$ under a map (called the encoding function) $\phi : \X \to \H$. The ``training'' step then associates each class with a vector in $\H$, which is typically formed by summing (bundling) the training data corresponding to a particular class. Specifically, class $j$ is represented as $\theta_{j} = \sum_{i=1}^{n}\alpha_{ij}\phi(x_{i})$
where $\alpha_{ij} = {1}(y_{i} = j)$. In practice, $\theta_{j}$ is sometimes quantized to reduce precision in some fashion, which is beneficial in some hardware settings. The label of a query $x$ is predicted by via $\hat{y} = \argmax_{j \in \{1,...,c\}} \phi(x) \cdot \theta_{j}$
where the operands are sometimes normalized if appropriate. In either case, this procedure can be interpreted as associating each class to a linear scoring function in $\H$ and performing inference by picking the class of highest score - a common paradigm in machine learning. Fine-tuning the class vectors $\theta_{j}$ is common, often achieved by running the Perceptron algorithm \cite{rosenblatt1958perceptron}, referred to as “re-training” in HDC literature.

%Concretely, given a misclassified sample $x$, the class prototype can be fine tuned via the following equation:
%\begin{equation}
%\begin{split}
%    \theta_{y} &= \theta_{y} + \phi(x) \\
%    \theta_{\hat{y}} &= \theta_{\hat{y}} %- \phi(x)
%\end{split}
%\end{equation}

A number of HDC encoding methods have been proposed to encode different types of data. For example, string or text data can be encoded through the $N$-gram encoding method~\cite{joshi2017language,imani2018hdna}. Concretely, let $S = (a_{1}a_{2}a_{3}...a_{N})$ be a string of $N$ characters drawn from some alphabet $\A$ (say, the latin alphabet or $\{\texttt{A,T,G,C}\}$). To encode $S$, we start by assigning each $a \in \A$ an embedding $\phi(a)$ by sampling uniformly at random from  $\{\pm 1/\sqrt{d}\}^{d}$, after which, we represent $S$ as $a_1 \circ \rho(a_2) \circ \rho^2(a_3) \circ \dots \circ \rho^{N-1}(a_N)$ where $\circ$ is element-wise multiplication~\cite{joshi2017language} and $\rho$ is a permutation operation of the vector coordinates, $\rho^n$ means same permutation applied $n$ times in sequence. If $\circ$ is multiplication, the way $N$-grams representations are constructed makes them almost mutually orthogonal in $\mathcal{H}$ in the sense that $\E[\phi(S)\cdot\phi(S')] = 1(S = S')$. The deviation from this expectation can be controlled using concentration arguments \cite{thomas2021theoretical}.

Finally, longer strings that contain multiple $N$-grams can be treated as a sum of $N$-gram vectors. This string encoding scheme can be viewed as a ``compressed'' version of the bag-of-words model~\cite{harris1954distributional} in the sense that all $N$-gram vectors are superimposed into one representation. Intuitively, the inner product between two such encoded strings can measure how ``similar'' two strings are, as that value would be large if two strings shared many common $N$-grams, and vice versa. For general feature vectors (or simple images), techniques based on random projection are popular~\cite{morris2021hydrea,khaleghi2022generic}. 

Existing HDC encoding methods tend to capture fairly simple notions of similarity based on the L1/L2 or angular distance. However, the design of good encoding functions for complex forms of data like graphs and time-series remains an important area of research. For inspiration, we turn to another area of machine learning that has thought extensively about how to measure similarities between data points using high-dimensional vectors.

\subsection{Kernel Methods}
\label{kernel_methods}

Kernel methods are a wide ranging area of research in statistics and machine learning that shares many similarities with HDC~\cite{shawe2004kernel,meanti2020kernel,hofmann2008kernel}. Much like in HDC, kernel methods work by embedding data into a high-dimensional space wherein similarities are measured using inner products. That is to say, kernel methods measure similarities between data points $x,x' \in \mathcal{X}$ via a function $K(x,x') = \psi(x) \cdot \psi(x')$, called a ``kernel function,'' where $\psi : \X \to \H$ is an embedding into an inner product space. For many types of kernel functions used in practice, it is possible to compute $K(x,x')$ directly on the ambient representation of the data without materializing the embeddings. Notable examples include the Gaussian kernel $K(x,x') \propto \exp(-\|x-x'\|_{2}^{2})$, and the $p$-th order polynomial kernel $K(x,x') = (1 + x \cdot x')^{p}$. Both of these kernels can be evaluated in closed form on the ambient representation of the data, allowing kernel methods to \emph{implicitly} compute a similarity based on a high-dimensional embedding. HDC, however, always explicitly  materializes the embeddings, hence the need for an encoding function $\phi$.

Kernel-based learning methods make predictions using functions taking the form $f(x) = \sum_{i=1}^{n} 
a_i k(x_{i},x)$, where $x_{1},...,x_{n}$ are training data points, and $\alpha_{1},...,\alpha_{n}$ are weights that are learned by a training algorithm. Noting that $f(x) = \sum_{i=1}^{n}\alpha_{i}k(x,x_{i}) = \psi(x)\cdot\theta$ where $\theta = \sum_{i=1}^{n} a_i \psi(x_i)$, in this way we can interpret such functions as \emph{linear models} in the embedding space associated with the kernel, much like in the previous paragraph on HDC. One significant difference between kernel methods and HDC, is that in the former the embeddings are implicit, and similarities are evaluated using the kernel function. This property is appealing because it allows one to efficiently work with infinite-dimensional embeddings, which can have desirable properties for learning~\cite{steinwart2001influence}.

\subsection{Kernel Methods and HDC}
\label{kernel_methods_HDC}

There is a large body of theoretical and applied literature on kernel methods that has developed kernel functions applicable to many settings of practical interest~\cite{neumann2016propagation,leslie2001spectrum,shimodaira2001dynamic,shawe2004kernel}. To provide a concrete example of how the literature on kernel methods can offer insights for the HDC community, we first consider the encoding of time-series data in HDC. Similar to the previously discussed $N$-gram encoding, the permutation operation is applied to encode the temporal information of time-series data~\cite{joshi2017language,asgarinejad2020detection}. However, such encoding schemes can fail if the events in two time-series do not align exactly. Since each time step is associated with a unique permutation during encoding, even a small shift in events between time-series can cause existing HDC encoding methods to map the two time-series to nearly orthogonal vectors. In practice, however, if two time-series reflect the same underlying activity or nature, one would expect their similarity to be preserved after encoding, even if some event misalignment exists. On this issue, the literature on kernel methods suggests a solution: the dynamic-time-warping kernel~\cite{gudmundsson2008support}, which can handle time-series sequences with misalignment or time-stretching/compression.

For graphs, GraphHD~\cite{nunes2022graphhd} proposes to encode graph topology induced by PageRank centrality metric~\cite{brin1998anatomy}. However, this process does not utilize crucial information such as node labels or node attributes (where each node in the graph is associated with a feature vector or a label). Such limitations have been addressed by kernel methods. For example, the propagation kernell~\cite{neumann2016propagation} works with graphs that include node labels or node attributes, therefore is capable of capturing a potentially richer notion of similarity.

In this work, our goal is to devise a procedure that can translate any kernel into an equivalent HDC encoding, thereby allowing practitioners to exploit the wealth of kernel functions that have been designed for practical problems while continuing to reap the benefits of computing with HDC representations.

%The propagation kernel uses random walk techniques to measure not only the structural difference between graphs but also taking node labels or node attributes into consideration as well, therefore is capable of capturing a potentially richer notion of similarity.

\subsection{Related Work}

The connection between HDC and kernel approximation is generally well known~\cite{yu2022understanding,thomas2021theoretical,paxon2021computing,voelker2020short}, mostly through the lens of random Fourier features (RFF)~\cite{rahimi2007random}. RFF is a sampling based scheme that generates a vector of features $\phi(x) \in \R^{d}$ with the property that $\phi(x) \cdot \phi(x') \approx K(x,x')$, where $K$ is a shift-invariant kernel (kernel function that depends only on the relative distance between inputs, e.g. the Gaussian kernel and Laplacian kernels). Closely related methods arise in the HDC literature under the names ``nonlinear-encoding'' \cite{miranda2022hyperdimensional,imani2020dual} and ``fractional power encoding'' \cite{paxon2021computing}. Using RFF in the context of HDC means that inner products in HD space approximate some shift-invariant kernel, usually the Gaussian or Sinc kernels. However, a limitation of RFF is that it can only work with shift-invariant kernels, which many useful kernels do not satisfy, such as kernels on graphs and strings. The \nys method provides a way to generate approximations for a larger class of kernels that do not need to be translation invariant.

% More recently, Cyclic Group Representation~\cite{yu2022understanding} has also explored the similarity-preserving properties of encodings. Specifically, the authors demonstrate that HDC encodings can be generated directly from a desired pairwise similarity matrix. However, their approach does not establish an explicit connection to kernel methods and is demonstrated entirely on Euclidean data, leaving room for generalization. Additionally, the method requires the computation and eigendecomposition of the entire similarity matrix, which can be computationally prohibitive in practice.

%More recently, the work by Yu et al.~\cite{yu2022understanding} proposed a similar scheme based on the decomposition of a predefined similarity matrix. Specifically, the authors demonstrated that HDC encodings can be generated directly from a desired pairwise similarity matrix. However, their approach does not establish an explicit connection to the Nyström method, a widely used technique in kernel approximation. Consequently, their method requires the computation and eigendecomposition of the entire similarity matrix, which can be computationally prohibitive in practice. Furthermore, it remains unclear whether their approach produces encodings that effectively approximate the underlying kernel. In this work, we demonstrate that our approach generates encodings that approximate the underlying kernel in a computationally efficient manner.

\section{HDC Encoding via \nys Approximation}

In this section, we describe our new encoding algorithm, NysHD, for HDC using the \nys method for kernel approximation.
We also demonstrate that the inner product between encoded samples, in expectation, preserves normalized kernel values.

Given a suitable kernel function $K$, our goal is to generate an encoding function $\phi: \mathcal{X} \rightarrow \mathcal{H}$ such that $\phi(x) \cdot \phi(x') \approx K(x,x') \quad \forall x,x' \in \mathcal{D}$, where $\mathcal{D}$ is some subset of $\mathcal{X}$. This property is useful for learning algorithms that are widely employed in HDC as it enables them to exploit more useful similarities captured by the kernel.

\subsection{\nys Method}

Conventional realizations of kernel-based learning algorithms commonly require storing all pairwise evaluations of the kernel function in a large matrix $G$ defined element-wise by $G_{ij} = K(x_{i},x_{j})$, which is problematic when $n$ is large. The \nys method is a low-rank matrix approximation technique widely employed to speed up kernel machines by avoiding the need to store the entire kernel matrix~\cite{williams2000using, drineas2005nystrom, kumar2012sampling}. In this way, the \nys method is similar to RFF since they are both sampling-based schemes and can be used to approximate kernel functions. The key difference between RFF and the \nys method is that the \nys method can work with a larger class of kernels than RFF, many of which are useful for applications involving discrete structures such as string and graph.

Intuitively, the \nys method works by sub-sampling the kernel matrix and reconstructing the full kernel matrix from the sampled one. This is possible because the kernel matrix is typically close to low-rank in practice. Concretely, suppose we have a dataset $\mathcal{D} = \{x_1, x_2, ... ,x_n\}$, from which we sample a set of landmarks $\mathcal{Z} = \{z_{1},...,z_{s}\}$, where $s \ll n$. Let $G \in \mathbb{R}^{n \times n}$ be the full kernel matrix defined element-wise by $G_{ij} = K(x_{i},x_{j})$, and let $H_{\mathcal{Z}} \in \mathbb{R}^{s \times s}$ be the sub-sampled kernel matrix defined element-wise by $(H_{\mathcal{Z}})_{ij} = K(z_{i},z_{j})$.
The \nys method yields the following approximation~\cite{drineas2005nystrom}:
\begin{equation}
    \hat{G} = C H_{\mathcal{Z}}^{+} C^T \approx G
    \label{eq:nys}
\end{equation}
Where $C$ is an $n \times s$ matrix such that $C_{ij} = K(x_i,z_j)$ for some kernel function $K$ and $H_Z^{+}$ denotes the pseudo-inverse of $H_{\mathcal{Z}}$. There is a robust theoretical literature on the \nys method, providing bounds on approximation error based on the number of selected landmarks and various sampling strategies~\cite{kumar2012sampling}. Let $Q$ and $\Lambda$ be the eigenvectors and eigenvalues of $H_{\mathcal{Z}}$ then:
\[
    H_Z^{+} = Q \Lambda^{-1} Q^T 
\]
This allows one to generate encodings explicitly via $\phi_{\text{nys}}(x_i) = \Lambda^{-\frac{1}{2}} Q^{T} C^{(i)}$, which approximates the kernel:
\begin{equation}
\begin{split}
    \label{nys_approx}
    & \phi_{\text{nys}}(x_i) \cdot \phi_{\text{nys}}(x_j) \\
    &= \left( \Lambda^{-\frac{1}{2}} Q^{T} C^{(i)} \right) \cdot \left( \Lambda^{-\frac{1}{2}} Q^{T} C^{(j)} \right) = \hat{G}_{ij}
\end{split} 
\end{equation}
Where $x_i, x_j \in \mathcal{D}$ and $C^{(i)}$ denotes the $i^{th}$ row of $C$ in a column vector. While the encodings produced by \nys method directly satisfy the desired kernel approximation property, they are, in general, of high-precision which is undesirable in some settings of interest in HDC~\cite{khaleghi2022generic}. Our method rectifies this issue by composing the features extracted using the \nys method with another encoding technique that preserves angular similarities which we discuss in detail in the next section.

Due to the data-dependent nature of the \nys method, the quality of its approximation and computational complexity depends on the size and composition of $\mathcal{Z}$. As an initial step, in this paper we simply use uniform sampling without replacement~\cite{williams2000using,kumar2012sampling} as our sampling strategy. However, more sophisticated strategies such as ensemble and adaptive sampling~\cite{kumar2012sampling} for constructing landmark sets can potentially lead to better performance and warrant further exploration in future work.

\subsection{Encoding Process}

\label{encoding}

To achieve the aforementioned goals, we compose random hyperplane rounding~\cite{charikar2002similarity} with the \nys method to generate HDC embeddings that approximate a desired kernel. Alg.~\ref{alg:adaptive} describes the process for generating the \nys embedding matrix, followed by data point encoding in Alg.~\ref{alg:encode}.

\begin{algorithm}
 \caption{Generate the \nys embedding matrix}
 \begin{algorithmic}  
 \REQUIRE kernel $K$ over $\mathcal{X}$, dataset $\mathcal{D}$, number of landmarks $s > 0$, HDC dimension $d > 0$
 \STATE $\mathcal{Z} \leftarrow$ sample $s$ points from $\mathcal{D}$ without replacement \\ \textbf{\textit{/*Landmarks*/}}
 \STATE $(H_{\mathcal{Z}})_{ij} = K(z_i,z_j) 
 \quad \forall \quad 0 \leq i,j \leq s$ \\ \textbf{\textit{/*Partial kernel Matrix over landmarks*/}}
 \STATE $Q \Lambda Q^{T} = H_{\mathcal{Z}}$ \\ \textbf{\textit{/*Symmetric Eigen-decomposition*/}}
 \STATE $P_{\text{rp}}$ = $[ w_1, w_1, \cdots, w_{d} ]^T \in \mathbb{R}^{d \times s}$ \\ \textbf{\textit{/*$w_i$ sampled from s dimensional unit sphere*/}}
 \STATE $P_{\text{nys}} = P_{\text{rp}} \Lambda^{-\frac{1}{2}} Q^{T}$
 \RETURN $P_{\text{nys}}$ , $\mathcal{Z}$
 \end{algorithmic}
 \label{alg:adaptive}
\end{algorithm}

{\begin{algorithm}
 \caption{Encode one data point from $\mathcal{D}$}
 \begin{algorithmic}  
 \REQUIRE $x_i \in \mathcal{D}$, \nys embedding matrix $P_{\text{\text{nys}}}$, Landmarks $\mathcal{Z}$ and kernel function $K$
 \STATE $C^{(i)} = \begin{bmatrix} K(x_i,z_1) & K(x_i,z_2) & \cdots & K(x_i,z_s) \end{bmatrix}^T \in \mathbb{R}^{s}$
 \RETURN $\sqrt{\frac{\pi}{2d}}$ sign $\left( P_{\text{nys}} C^{(i)} \right)$
 \end{algorithmic}
 \label{alg:encode}
\end{algorithm}}

The rows of $P_{\text{rp}} \in \mathbb{R}^{d \times s}$ are sampled from the uniform distribution over the $s$-dimensional unit sphere. This enables the following result regarding sign-thresholded random projection: suppose $v,v' \in \mathbb{R}^n$ are unit vectors, with respect to randomness in the sampling of $P_{\text{rp}}$, the following result holds with randomness in the sampling of $P_{\text{rp}}$ (see for instance \cite{charikar2002similarity}):
\begin{equation}
\begin{split}
  & \mathbb{E} \left[ \frac{1}{d} \text{sign} (P_{\text{rp}}v)\cdot \text{sign} (P_{\text{rp}}v') \right] = (1 - 2\cos^{-1}(v\cdot v')/\pi)
  \label{eq:rp}
\end{split}
\end{equation}
NysHD generates encodings for which the similarity induced by the kernel $K$ is preserved by dot product between the encodings in $\mathcal{H}$. We summarize this result in the following theorem:

\begin{theorem}
\label{main_theorem}

Let $K$ be a positive-definite kernel and, using the notation of Algorithms \ref{alg:adaptive} and \ref{alg:encode}, and recalling that $\phi_{\text{nys}}(x_i) = \Lambda^{-\frac{1}{2}} Q^{T} C^{(i)}$, define:
\[
    \phi(x_i) = \sqrt{\frac{\pi}{2d}} \text{sign}(P_{\text{rp}} \phi_{\text{nys}}(x_i))
\]
Then, for all $x_{i}, x_{j} \in \mathcal{D}$:
\begin{equation}
    \mathbb{E} \left[ \phi \left(x_{i} \right) \cdot \phi\left(x_{j} \right) \right] = \frac{\pi}{2} - \cos^{-1} \left( \frac{\hat{G}_{ij}}{\sqrt{\hat{G}_{ii} \hat{G}_{jj}}} \right)
\end{equation}
Where $\hat{G}_{ij}$ is estimated kernel value between $x_i$ and $x_j$ produced by \nys method. The expectation is taken with respect to randomness in the sampling of landmarks and $P_{\text{rp}}$.

\end{theorem}
\textbf{Proof:} Let $\bar{\phi}_{\text{nys}}(x_i)$ denote normalization: $\frac{\phi_{\text{nys}}(x_i)}{||\phi_{\text{nys}}(x_i)||}$. Noting that $ \text{sign} (c v) = \text{sign} (v) $ for any $v$ if $c \geq 0$:
\begin{equation}
\begin{split}
& \phi \left( x_i \right) \cdot \phi \left( x_j \right) \\ 
&= \frac{\pi}{2d} \left( \text{sign} \left( P_{\text{rp}} \phi_{\text{nys
}} \left(x_i \right) \right) \cdot  \text{sign} \left( P_{\text{rp}} \phi_{\text{nys
}}(x_j) \right) \right) \\
& = \frac{\pi}{2} \left( \frac{1}{d} \text{sign} \left( P_{\text{rp}} \bar{\phi}_{\text{nys}}(x_i) \right) \cdot \text{sign} \left( P_{\text{rp}} \bar{\phi}_{\text{nys}}(x_j) \right) \right) \\
\end{split}
\end{equation}
Using result regarding sign-thresholded random projection in equation~\ref{eq:rp}, we have:
\begin{equation}
\begin{split}
    & \mathbb{E} \left[ \phi(x_i) \cdot \phi(x_j) \right] \\
    &= \frac{\pi}{2} \left(1- \frac{2 \cos^{-1}\left({\bar{\phi}_{\text{nys}}(x_i
    )} \cdot \bar{\phi}_{\text{nys}}(x_j) \right) }{\pi} \right) \\
\end{split}
\end{equation}
Recalling the \nys method in equation ~\ref{nys_approx} and the fact that $||{\phi_{\text{nys}}(x_i)}|| = \sqrt{{\phi_{\text{nys}}(x_i)} \cdot {\phi_{\text{nys}}(x_i)}} = \sqrt{\hat{G}_{ii}}$:
\begin{equation}
\begin{split}
    \bar{\phi}_{\text{nys}}(x_i) \cdot \bar{\phi}_{\text{nys}}(x_j) = \frac{\hat{G}_{ij}}{\sqrt{\hat{G}_{ii} \hat{G}_{jj}}}
\end{split}
\end{equation}
Finally:
\begin{equation}
\begin{split}
    \mathbb{E} \left[ \phi(x_i) \cdot \phi(x_j) \right] &= \frac{\pi}{2} - \cos^{-1} \left( \frac{\hat{G}_{ij}}{\sqrt{\hat{G}_{ii} \hat{G}_{jj}}} \right) \blacksquare 
\end{split}
\end{equation}

To make the relationship with kernel approximation more explicit, consider the first order Taylor expansion of $\cos^{-1}$. Since $||\bar{\phi}_{\text{nys}}(x_i)|| = ||\bar{\phi}_{\text{nys}}(x_j)|| = 1$, it follows that $-1 \leq {\bar{\phi}_{\text{nys}}(x_i)} \cdot \bar{\phi}_{\text{nys}}(x_j) \leq 1$, which means $\frac{\hat{G}_{ij}}{\sqrt{\hat{G}_{ii} \hat{G}_{jj}}}$ is within the domain of $\cos^{-1}$, So:
\begin{equation}
\begin{split}
    \mathbb{E} \left[ \phi(x_i) \cdot \phi(x_j) \right] &\approx \frac{\pi}{2} - \left( \frac{\pi}{2} - \frac{\hat{G}_{ij}}{\sqrt{\hat{G}_{ii} \hat{G}_{jj}}} \right) \\
    & = \frac{\hat{G}_{ij}}{\sqrt{\hat{G}_{ii} \hat{G}_{jj}}}
\end{split}
\end{equation}

Note that this is called normalized kernel~\cite{ah2010normalized}.
As shown above, NysHD preserves the kernel in $\mathcal{H}$ (up to the first order approximation) in the sense that the inner product between any pair of encoded data points approximates the normalized kernel value of some user-defined kernel $K$. 
Since HDC relies on the inner product in $\mathcal{H}$ for inference, as described in Section ``Learning with HD''~\ref{2.1}, the similarity metric captured by the kernel $K$ is preserved, which benefits subsequent HDC learning algorithms.

\subsection{Encoding Complexity}

\label{complexity}
The proposed encoding method differs from existing HDC methods because generating embeddings using the \nys method requires computing a partial kernel matrix. For this reason, similar to kernel machines, the efficiency of our algorithm depends heavily on the number of kernel computations, which is related to the size of landmarks set $s$.

In this section, we provide an asymptotic runtime analysis of the encoding algorithm. Assuming that evaluating the kernel function $K$ takes $\mathcal{O}(m)$, that is to say the complexity of kernel function is linear in the dimension of the input (e.g., the Gaussian kernel) - the complexity of running the encoding algorithm over a dataset of $n$ samples with $s$ landmarks to $d$ dimensional embedding is $O(\max{(s^{3}, snm, snd)})$. The first term corresponds to the symmetric eigen-decomposition of the matrix $H_{\mathcal{Z}}$, the second term reflects the number of kernel evaluations, and the last term represents the sign-thresholded random projection, as discussed previously. In a typical \nys kernel approximation setting where $s \ll n$, depending on the nature of the kernel function being used, the second term is likely to be the dominant term, which is why the choice of $s$ and $K$ is crucial. Moreover, the analysis we provided here assumes the complexity of the kernel is linear in the number of features (for example, the Gaussian kernel); however, this is not always the case for kernel functions. For example, the spectrum kernel~\cite{leslie2001spectrum}, which compares the $N$-gram composition between two input strings, has $\mathcal{O}(m \log (m))$ in the length of the input sequences.

%%%%%%%%%%%%%%%%%%%%%%%%%%%%%%%%%%%%%%%%%%%%%%%%%%%%%%%%%%
\section{Evaluation}

We conduct an empirical evaluation of NysHD. First, we validate \ref{main_theorem} and then evaluate the practicality of our method in terms of accuracy and efficiency against both HDC and non-HDC baselines.

\subsection{Datasets and Kernel Functions}
\label{kernel}

\begin{table*}[h]
\centering
\small
\begin{tabular}{ccccccc}
\hline
\textbf{Task} & \textbf{dataset} & \textbf{\# training} & \textbf{\# testing} & \textbf{\# class} & \textbf{Description} \\
\hline
Graph & \textbf{ENZYMES}~\cite{borgwardt2005protein} & 480 & 120 & 6 & Graph with attributed nodes\\
 & \textbf{NCI1} \cite{wale2008comparison} & 3288 & 822 & 2 & Graph with labeled nodes\\
 & \textbf{D\&D}~\cite{dobson2003distinguishing} & 943 & 235 & 2 & Graph with labeled nodes\\
 & \textbf{BZR}~\cite{sutherland2003spline} & 324 & 81 & 2 & Graph with attributed nodes \\
 & \textbf{MUTAG}~\cite{debnath1991structure} & 150 & 38 & 2 & Graph with labeled nodes \\
 & \textbf{COX2}~\cite{sutherland2003spline} & 373 & 94 & 2 & Graph with attributed nodes \\
 & \textbf{NCI109}~\cite{wale2008comparison} & 3301 & 826 & 2 & Graph with labeled nodes \\
 & \textbf{Mutagenicity}~\cite{riesen2008iam} & 3469 & 868 & 2 & Graph with labeled nodes \\
\hline
 String & \textbf{Protein}~\cite{selvaraj2023ion} & 721 & 181 & 6 & Protein sequence \\
 & \textbf{SMS}~\cite{misc_sms_spam_collection_228} & 4459  & 1115 & 2 & Nautral Language \\
 & \textbf{Splice}~\cite{towell1992primate} & 2552 & 628 & 3 & DNA sequence\\
 & \textbf{Promoter}~\cite{misc_molecular_biology} & 84 & 22 & 2 & DNA sequence\\
\hline
\hline
\end{tabular}
\caption{Summary of tasks and datasets}
\label{tbl:dataset}
\end{table*}

The proposed encoding method is general-purpose and applicable to various data and tasks, provided a suitable kernel function is available. To confirm its versatility, we conduct assessments across two distinct tasks: \textbf{Graphs} and \textbf{Strings} classification. In total, we have chosen 8 graph datasets from TUDataset~\cite{morris2020tudataset}, a well-known  standardized repository for graph classification benchmarking datasets, and 4 string datasets for bio-sequence and text classification. More information on datasets can be found in Table.~\ref{tbl:dataset}.

The main advantage of our method is that it generates embeddings for HDC learning algorithms that preserve any user-defined positive-semidefinite kernel function. As such, the choice of kernel function is crucial for achieving the best accuracy and efficiency. For our method, we use the gappy kernel~\cite{leslie2004fast} for string classification. The gappy kernel is a variant of the spectrum kernel~\cite{leslie2001spectrum} that allows a small amount of gap within the $N$-grams. The propagation kernel~\cite{neumann2016propagation} is used for graph classification for its ability to work with labeled or attributed graphs, as discussed in the section ``kernel methods"~\ref{kernel_methods}. We chose the aforementioned kernels for their relatively low computation complexity and effectiveness on respective tasks.

\subsection{Experimental Setup and Baselines}
\label{setup}

We benchmark our method against the existing HDC encoding baselines within each domain: \textbf{Graph classification}  using the encoding scheme from introduced in GraphHD~\cite{nunes2022graphhd} and \textbf{String classification} uses $N$-gram HDC encoding approach~\cite{joshi2017language}. The details of both HDC encoding methods are discussed in the background section.

In addition to HDC-baselines, we also include comparisons with popular state-of-the-art deep neural network architectures. For graph classification, we use \textbf{DGCNN}~\cite{zhang2018end}, \textbf{GCN}~\cite{chen2019powerful}, \textbf{GIN}~\cite{xu2018powerful}, \textbf{GIUNet}~\cite{amouzad2024graph}. On string datasets, following the recent trend of applying large models for bio-sequence and language modeling~\cite{qiu2020pre,raffel2020exploring}, we fine-tune Large protein model (\textbf{ESM-2-8m}~\cite{rives2021biological}) for for bio-sequence datasets, and large language model (\textbf{BERT}~\cite{devlin2018bert}) for natural language dataset.

For our method, we set the number of landmarks: $s$ = $max$(300, 2\% of training data) on each dataset, and kernel specific hyperparamters are chosen empirically. To ensure the fairness of comparison, we use an identical HDC learning pipeline adapted from OnlineHD~\cite{hernandez2021onlinehd} when evaluating different HDC encoding methods. The Perceptron algorithm \cite{rosenblatt1958perceptron} is used to fine-tune class prototypes. In all of our experiments, 20 epochs of fine-tuning have been found to be sufficient.

All experiments were run on an Intel i5-11400 CPU (except for large models fine-tuning, which required an Nvidia RTX 3050 GPU is used due to long training time on CPU). We evaluate different methods by their training efficiency (time in seconds) and classification accuracy. Each experiment was run 10 times, and we report the mean and standard deviation of the results.

\subsection{Approximating Normalized Kernel Matrix}
\label{approx}

\begin{figure*}
\centering
  \includegraphics[width=1\textwidth]{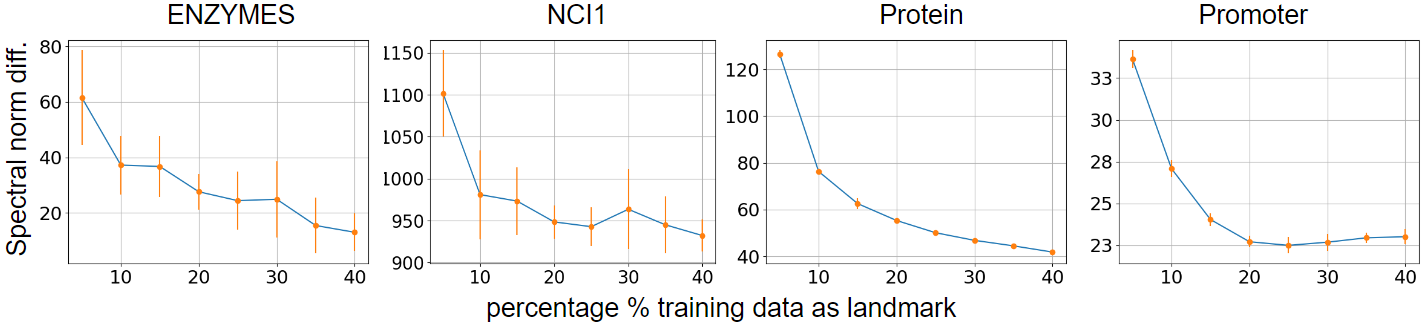}
  \caption{Numerical difference of spectral norm between normalized kernel matrices compute directly from kernel function and approximated kernel matrices with our encoding.}
\label{kernelapprox}
\end{figure*}

The classification accuracy of HDC-based models hinges on the capability of encodings to capture some salient notion of similarity via inner product. 
A straightforward way to verify our encodings preserve the kernel is to compare the spectral norm of the normalized kernel matrix (computed using the kernel function) with the pairwise inner products of the encodings. Using different percentages of the training samples as landmarks, the results for selected datasets (ENZYMES, NCI1, Protein, Promoter) are shown in Fig.~\ref{kernelapprox}. In line with Theorem \ref{main_theorem}, our method preserves the normalized kernel value, whose quality is positively correlated with the number of landmarks (a less noisy approximation can be achieved with a larger number of landmarks). The results here indicate that our method is effective in transferring similarity functions in kernel methods to HDC settings.

\subsection{Accuracy and Efficiency Results}
\label{result}

\begin{table*}[h]
\renewcommand{\arraystretch}{1.1}
\setlength{\tabcolsep}{3pt}
\small
\centering
\begin{tabular}{cc|cccccccc}
\hline
\hline
\textbf{Method} && NCI1 & ENZYMES & D\&D
 & BZR & MUTAG & COX2 & NCI109 & Mutagenicity \\
\hline
\textbf{DGCNN} &{Acc.}& 70.2 $\pm$ 2.2\% & 36.9 $\pm$ 4.6\% & 74.5 $\pm$ 3.4\% & 81.5 $\pm$ 4.4\% & 82.9 $\pm$ 4.1\% & 78.3 $\pm$ 2.7\% & 71.1 $\pm$ 2.2\% & 75.0 $\pm$ 1.3\% \\

&\textit{Time}& \textit{37.5} $\pm$ \textit{1.2s} & \textit{5.6} $\pm$ \textit{0.5s} & \textit{59.2} $\pm$ \textit{3.0s} & \textit{4.0} $\pm$ \textit{0.0s} & \textit{1.0} $\pm$ \textit{0.0s} & \textit{4.1} $\pm$ \textit{0.3s} & \textit{36.2} $\pm$ \textit{0.9s} & \textit{44.8} $\pm$ \textit{1.3s} \\

\hline
\textbf{GCN} &{Acc.}& \textbf{79.9} $\pm$ \textbf{1.1\%} & \underline{60.7 $\pm$ 2.1\%} & \underline{74.8 $\pm$ 1.6\%} & \textbf{84.2} $\pm$ \textbf{2.2\%} & \textbf{85.5} $\pm$ \textbf{3.8\%} & \textbf{83.1} $\pm$ \textbf{3.7\%} & \textbf{80.2} $\pm$ \textbf{1.1\%} & \textbf{79.9} $\pm$ \textbf{1.0\%} \\

&\textit{Time}& \textit{74.3} $\pm$ \textit{0.6s} & \textit{15.9} $\pm$ \textit{0.5s} & \textit{262.3} $\pm$ \textit{9.2s} & \textit{7.9} $\pm$ \textit{0.3s} & \textit{3.0} $\pm$ \textit{0.0s} & \textit{11.0} $\pm$ \textit{0.0s} & \textit{73.5} $\pm$ \textit{0.7s} & \textit{75.3} $\pm$ \textit{0.5s} \\

 \hline
\textbf{GIN} &{Acc.}& 73.4 $\pm$ 1.7\% & 28.8 $\pm$ 4.2\% & 67.5 $\pm$ 5.9\%
 & 76.4 $\pm$ 8.4\% & 76.8 $\pm$ 9.6\% & 78.1 $\pm$ 4.3\% & 70.8 $\pm$ 2.6\% & \underline{78.0 $\pm$ 1.7\%} \\

&\textit{Time}& \textit{66.8} $\pm$ \textit{1.0s} & \textit{9.0} $\pm$ \textit{0.0s} & \textit{63.6} $\pm$ \textit{0.7s} & \textit{6.0} $\pm$ \textit{0.0s} & \textit{2.0} $\pm$ \textit{0.0s} & \textit{7.2} $\pm$ \textit{0.4s} & \textit{67.1} $\pm$ \textit{0.5s} & \textit{70.5} $\pm$ \textit{0.5s} \\

 \hline
\textbf{GIUNet} &{Acc.}& 72.4 $\pm$ 1.4\% & 29.9 $\pm$ 4.0\% & 63.4 $\pm$ 6.9\%
 & 78.3 $\pm$ 10.7\% & 85.0 $\pm$ 4.7\% & 77.4 $\pm$ 7.0\% & 68.8 $\pm$ 4.3\% & 76.5 $\pm$ 0.8\% \\
 
&\textit{Time}& \textit{89.0} $\pm$ \textit{0.4s} & \textit{13.0} $\pm$ \textit{0.0s} & \textit{96.2} $\pm$ \textit{0.7s} & \textit{9.0} $\pm$ \textit{0.0s} & \textit{3.0} $\pm$ \textit{0.0s} & \textit{11.0} $\pm$ \textit{0.0s} & \textit{89.5} $\pm$ \textit{0.5s} & \textit{94.6} $\pm$ \textit{0.5s} \\

 \hline
\textbf{GraphHD} &{Acc.}& 60.0 $\pm$ 2.1\% & 23.2 $\pm$ 2.3\% & 67.6 $\pm$ 1.3\%
 & 74.9 $\pm$ 2.1\% & \underline{85.3 $\pm$ 10.2\%} & \underline{81.9 $\pm$ 3.9\%} & 59.9 $\pm$ 2.1\% & 59.8 $\pm$ 1.7\% \\

&\textit{Time}& \textit{30.0} $\pm$ \textit{0.7s} & \textit{6.0} $\pm$ \textit{0.0s} & \textit{32.9} $\pm$ \textit{0.3s} & \textit{3.0} $\pm$ \textit{0.0s} & \textit{1.0} $\pm$ \textit{0.0s} & \textit{3.0} $\pm$ \textit{0.0s} & \textit{30.0} $\pm$ \textit{0.0s} & \textit{31.3} $\pm$ \textit{0.4s} \\

\hline
\textbf{NysHD} &{Acc.}& \underline{73.8 $\pm$ 2.2\%} & \textbf{61.3} $\pm$ \textbf{2.5\%} & \textbf{76.2} $\pm$ \textbf{2.8\%}
 & \underline{82.0 $\pm$ 3.3\%} & \textbf{85.5} $\pm$ \textbf{3.4\%} & 74.6 $\pm$ 3.7\% & \underline{71.8 $\pm$ 2.0\%} & 75.2 $\pm$ 0.9\% \\

&\textit{Time}& \textit{43.8} $\pm$ \textit{0.8s} & \textit{7.2} $\pm$ \textit{0.6s} & \textit{35.4} $\pm$ \textit{1.6s} & \textit{3.4} $\pm$ \textit{0.5s} & \textit{2.0} $\pm$ \textit{0.0s} & \textit{5.5} $\pm$ \textit{0.8s} & \textit{44.4} $\pm$ \textit{0.8s} & \textit{35.7} $\pm$ \textit{1.0s} \\
 
\hline
\hline
\end{tabular}
\caption{Experimental results. The best accuracy result for each dataset are \textbf{highted} and second best are \underline{underlined}.}
\label{tbl:g_result}
\end{table*}

\begin{table}
\renewcommand{\arraystretch}{1.1}
\setlength{\tabcolsep}{1mm} 
\small
\centering
\begin{tabular}{c|cccc}
\hline
\hline
\textbf{Method}  & Protein & SMS & Promoter & Splice\\
\hline
\textbf{LM} & 96 $\pm$ 0\% & \textbf{99 $\pm$ 0\%} & \underline{82 $\pm$ 2.7\%} & \textbf{92 $\pm$ 1.4\%} \\
\textbf{Finetune} & \textit{2003 $\pm$ 193s} & \textit{2846 $\pm$ 31s} & \textit{6 $\pm$ 3.0s} & \textit{142 $\pm$ 1.9s} 
\\

\hline
\textbf{$N$-gram} & \underline{96 $\pm$ 0.8\%} & \underline{97 $\pm$ 0.3}\% & 52 $\pm$ 4.0\% & 43 $\pm$ 4.9\% \\
\textbf{HDC} & \textit{13 $\pm$ 0.0s} & \textit{43 $\pm$ 0.4s} & \textit{1 $\pm$ 0.0s} & \textit{25 $\pm$ 0.8s} \\

\hline
\textbf{NysHD} & \textbf{98 $\pm$ 0.3\%} & \underline{97 $\pm$ 0.2\%} & \textbf{89 $\pm$ 3.4\%} & \underline{72 $\pm$ 1.7\%} \\
 & \textit{19 $\pm$} 0.7s & \textit{34 $\pm$ 0.0s} & \textit{1 $\pm$ 0.0s} & \textit{21 $\pm$ 0.3s} \\
\hline
\hline
\end{tabular}
\caption{Accuracy and training time on string datasets}
\label{tbl:s_result}
\end{table}

The accuracy results for graph and string datasets are summarized in Table~\ref{tbl:g_result} and Table~\ref{tbl:s_result}. In comparison with the GraphHD, NysHD achieves, on average, 11\% better accuracy on graph datasets. The improvement is especially significant on the ENZYMES dataset, which shows a 38\% accuracy improvement. This is because the previous HDC encoding method on graph~\cite{nunes2022graphhd} could not utilize node attributes in attributed graphs, whereas our method can. On string datasets, our method again consistently achieves better accuracy compared to $N$-gram based HDC encoding, with an average improvement of 17\%. Efficiency-wise, our method is comparable to previous HDC encoding methods,  although it is slightly slower than GraphHD~\cite{nunes2022graphhd} on graph datasets due to kernel computation.

Despite the efficiency and implementation simplicity of HDC learning algorithms, there remains an accuracy gap between HDC-based models and deep neural network (DNN) models for more complex classification tasks. We hope this work will help reduce that gap. To that end, we also include several state-of-the-art DNN approaches on both graph and string classification in our comparison. On graph datasets, NysHD achieved the best accuracy in 3 out of 8 graph datasets. On average, NysHD outperforms DGCNN by 3\%, GIN and GIUNet by 6\% for graph classification. Although graph convolutional network (GCN) still yield the best accuracy for some datasets, our method is on average 52\% faster than GCN. On string datasets, our method with the gappy kernel achieves better accuracy on 2 out of 3 bio-sequence datasets. For SMS and Splice dataset, the LM fine-tune method achieves better classification accuracy, but their training time is 6 to 83 times longer than our method.

Overall, the strength of NysHD becomes most apparent when dealing with data with complex structures or attributes that existing HDC encoding methods do not exploit.

%%%%%%%%%%%%%%%%%%%%%%%%%%%%%%%%%%%%%%%%%%%%%%%%%%%%%%%%%%%%%%
\section{Overhead \& Scalability}

This work allows future HDC works to exploit the power of kernel methods while still conforming to the general formalism and benefits of HDC. We recognize that the improvements in NysHD also come with additional computation costs in the form of kernel evaluation. How to minimize such cost for HDC applications are non-trivial problems that need further investigations. The main scalability challenge is to obtain a set of landmarks that is as small as possible, while still providing a good approximation to the true kernel matrix. There is a large body of work on more sophisticated sampling schemes for the \nys method that could help make our methods scalable to larger datasets~\cite{kumar2012sampling,musco2017recursive}. We would be interested in studying these in future work.

\section{Conclusion}

The success of HDC-based learning methods is contingent upon identifying an encoding function that preserves a suitable notion of similarity (kernel) for the task at hand. In this paper, we leverage the connection between the kernel method and HDC through the lens of \nys method for kernel estimation. Particularly, we propose NysHD, a new HDC encoding method that constructs encoding functions using suitable kernel functions for specific tasks.
As a result, compared with previous HDC encoding methods, NysHD achieves substantial improvements - on average, 11\% accuracy improvement on graph datasets and 17\% on string datasets. There are many situations in which methods from the kernel literature outperform existing HDC-based solutions, therefore, our approach can be expected to lead to performance improvements in many HDC applications.

\section{Acknowledgements}
This work was supported in part by National Science Foundation under Grants \#2003279, \#1826967, \#2100237, \#2112167, \#1911095, \#2112665, and in part by SRC under task \#3021.001. This work was also supported in part by PRISM and CoCoSys, centers in JUMP 2.0, an SRC program sponsored by DARPA.

\bibliography{aaai25}

@book{shawe2004kernel,
  title={Kernel methods for pattern analysis},
  author={Shawe-Taylor, John and Cristianini, Nello},
  year={2004},
  publisher={Cambridge university press}
}

@book{smola1998learning,
  title={Learning with kernels},
  author={Smola, Alex J and Sch{\"o}lkopf, Bernhard},
  volume={4},
  year={1998},
  publisher={Citeseer}
}

@Article{	  rahimi2018efficient,
  title		= {Efficient biosignal processing using hyperdimensional
		  computing: Network templates for combined learning and
		  classification of {ExG} signals},
  author	= {Rahimi, Abbas and Kanerva, Pentti and Benini, Luca and
		  Rabaey, Jan M},
  journal	= {Proceedings of the IEEE},
  volume	= {107},
  number	= {1},
  pages		= {123--143},
  year		= {2018},
  publisher	= {IEEE}
}

@article{menon2022highly,
  title={A highly energy-efficient hyperdimensional computing processor for biosignal classification},
  author={Menon, Alisha and Sun, Daniel and Sabouri, Sarina and Lee, Kyoungtae and Aristio, Melvin and Liew, Harrison and Rabaey, Jan M},
  journal={TBCAS},
  year={2022},
  publisher={IEEE}
}

@book{hertz2018introduction,
  title={Introduction to the theory of neural computation},
  author={Hertz, John A},
  year={2018},
  publisher={Crc Press}
}

@article{steinwart2001influence,
  title={On the influence of the kernel on the consistency of support vector machines},
  author={Steinwart, Ingo},
  journal={JMLR},
  volume={2},
  number={Nov},
  year={2001}
}

@article{plate1995holographic,
  title={Holographic reduced representations},
  author={Plate, Tony A},
  journal={IEEE Transactions on Neural networks},
  volume={6},
  number={3},
  pages={623--641},
  year={1995},
  publisher={IEEE}
}

@article{kanerva2009hyperdimensional,
  title={Hyperdimensional computing: An introduction to computing in distributed representation with high-dimensional random vectors},
  author={Kanerva, Pentti},
  journal={Cognitive computation},
  volume={1},
  pages={139--159},
  year={2009},
  publisher={Springer}
}

@article{thomas2021theoretical,
  title={A theoretical perspective on hyperdimensional computing},
  author={Thomas, Anthony and Dasgupta, Sanjoy and Rosing, Tajana},
  journal={JAIR},
  volume={72},
  pages={215--249},
  year={2021}
}

@inproceedings{khaleghi2022generic,
  title={GENERIC: highly efficient learning engine on edge using hyperdimensional computing},
  author={Khaleghi, Behnam and Kang, Jaeyoung and Xu, Hanyang and Morris, Justin and Rosing, Tajana},
  booktitle={DAC},
  pages={1117--1122},
  year={2022}
}

@inproceedings{zhao2022fedhd,
  title={FedHD: federated learning with hyperdimensional computing},
  author={Zhao, Quanling and Lee, Kai and Liu, Jeffrey and Huzaifa, Muhammad and Yu, Xiaofan and Rosing, Tajana},
  booktitle={MobiCom},
  pages={791--793},
  year={2022}
}

@inproceedings{morris2021hydrea,
  title={Hydrea: Towards more robust and efficient machine learning systems with hyperdimensional computing},
  author={Morris, Justin and Ergun, Kazim and Khaleghi, Behnam and Imani, Mohsen and Aksanli, Baris and Rosing, Tajana},
  booktitle={DATE},
  pages={723--728},
  year={2021},
  organization={IEEE}
}

@inproceedings{dutta2022hdnn,
  title={Hdnn-pim: Efficient in memory design of hyperdimensional computing with feature extraction},
  author={Dutta, Arpan and Gupta, Saransh and Khaleghi, Behnam and Chandrasekaran, Rishikanth and Xu, Weihong and Rosing, Tajana},
  booktitle={GLSVLSI},
  pages={281--286},
  year={2022}
}

@inproceedings{kang2022xcelhd,
  title={Xcelhd: An efficient gpu-powered hyperdimensional computing with parallelized training},
  author={Kang, Jaeyoung and Khaleghi, Behnam and Kim, Yeseong and Rosing, Tajana},
  booktitle={ASP-DAC},
  pages={220--225},
  year={2022},
  organization={IEEE}
}

@article{drineas2005nystrom,
  title={On the Nystr{\"o}m Method for Approximating a Gram Matrix for Improved Kernel-Based Learning.},
  author={Drineas, Petros and Mahoney, Michael W and Cristianini, Nello},
  journal={JMLR},
  volume={6},
  number={12},
  year={2005}
}

@article{neumann2016propagation,
  title={Propagation kernels: efficient graph kernels from propagated information},
  author={Neumann, Marion and Garnett, Roman and Bauckhage, Christian and Kersting, Kristian},
  journal={Machine learning},
  volume={102},
  pages={209--245},
  year={2016},
  publisher={Springer}
}

@incollection{leslie2001spectrum,
  title={The spectrum kernel: A string kernel for SVM protein classification},
  author={Leslie, Christina and Eskin, Eleazar and Noble, William Stafford},
  booktitle={Biocomputing 2002},
  pages={564--575},
  year={2001},
  publisher={World Scientific}
}

@article{hofmann2008kernel,
  title={Kernel methods in machine learning},
  author={Hofmann, Thomas and Sch{\"o}lkopf, Bernhard and Smola, Alexander J},
  year={2008}
}

@inproceedings{gudmundsson2008support,
  title={Support vector machines and dynamic time warping for time series},
  author={Gudmundsson, Steinn and Runarsson, Thomas Philip and Sigurdsson, Sven},
  booktitle={2008 IEEE International Joint Conference on Neural Networks},
  pages={2772--2776},
  year={2008},
  organization={IEEE}
}

@inproceedings{imani2018hdna,
  title={Hdna: Energy-efficient dna sequencing using hyperdimensional computing},
  author={Imani, Mohsen and Nassar, Tarek and Rahimi, Abbas and Rosing, Tajana},
  booktitle={BHI},
  pages={271--274},
  year={2018},
  organization={IEEE}
}

@inproceedings{asgarinejad2020detection,
  title={Detection of epileptic seizures from surface eeg using hyperdimensional computing},
  author={Asgarinejad, Fatemeh and Thomas, Anthony and Rosing, Tajana},
  booktitle={EMBC},
  pages={536--540},
  year={2020},
  organization={IEEE}
}

@article{paxon2021computing,
  title={Computing on Functions Using Randomized Vector Representations},
  author={Paxon Frady, E and Kleyko, Denis and Kymn, Christopher J and Olshausen, Bruno A and Sommer, Friedrich T},
  journal={},
  pages={arXiv--2109},
  year={2021}
}

@article{rahimi2007random,
  title={Random features for large-scale kernel machines},
  author={Rahimi, Ali and Recht, Benjamin},
  journal={NeurIPS},
  volume={20},
  year={2007}
}

@inproceedings{hernandez2021onlinehd,
  title={Onlinehd: Robust, efficient, and single-pass online learning using hyperdimensional system},
  author={Hern{\'a}ndez-Cano, Alejandro and Matsumoto, Namiko and Ping, Eric and Imani, Mohsen},
  booktitle={DATE},
  pages={56--61},
  year={2021},
  organization={IEEE}
}

@inproceedings{nunes2022graphhd,
  title={GraphHD: Efficient graph classification using hyperdimensional computing},
  author={Nunes, Igor and Heddes, Mike and Givargis, Tony and Nicolau, Alexandru and Veidenbaum, Alex},
  booktitle={DATE},
  pages={1485--1490},
  year={2022},
  organization={IEEE}
}

@article{brin1998anatomy,
  title={The anatomy of a large-scale hypertextual web search engine},
  author={Brin, Sergey and Page, Lawrence},
  journal={Computer networks and ISDN systems},
  volume={30},
  number={1-7},
  pages={107--117},
  year={1998},
  publisher={Elsevier}
}

@misc{misc_sms_spam_collection_228,
  author       = {Almeida,Tiago and Hidalgo,Jos},
  title        = {{SMS Spam Collection}},
  year         = {2012},
  howpublished = {UCI Machine Learning Repository},
  note         = {{DOI}: https://doi.org/10.24432/C5CC84}
}

@article{selvaraj2023ion,
  title={Ion-pumping microbial rhodopsin protein classification by machine learning approach},
  author={Selvaraj, Muthu Krishnan and Thakur, Anamika and Kumar, Manoj and Pinnaka, Anil Kumar and Suri, Chander Raman and Siddhardha, Busi and Elumalai, Senthil Prasad},
  journal={BMC bioinformatics},
  volume={24},
  number={1},
  pages={29},
  year={2023},
  publisher={Springer}
}

@article{wale2008comparison,
  title={Comparison of descriptor spaces for chemical compound retrieval and classification},
  author={Wale, Nikil and Watson, Ian A and Karypis, George},
  journal={Knowledge and Information Systems},
  volume={14},
  pages={347--375},
  year={2008},
  publisher={Springer}
}

@article{borgwardt2005protein,
  title={Protein function prediction via graph kernels},
  author={Borgwardt, Karsten M and Ong, Cheng Soon and Sch{\"o}nauer, Stefan and Vishwanathan, SVN and Smola, Alex J and Kriegel, Hans-Peter},
  journal={Bioinformatics},
  volume={21},
  year={2005},
  publisher={Oxford University Press}
}

@article{meanti2020kernel,
  title={Kernel methods through the roof: handling billions of points efficiently},
  author={Meanti, Giacomo and Carratino, Luigi and Rosasco, Lorenzo and Rudi, Alessandro},
  journal={NeurIPS},
  volume={33},
  pages={14410--14422},
  year={2020}
}

@article{kumar2012sampling,
  title={Sampling methods for the Nystr{\"o}m method},
  author={Kumar, Sanjiv and Mohri, Mehryar and Talwalkar, Ameet},
  journal={JMLR},
  volume={13},
  number={1},
  pages={981--1006},
  year={2012},
  publisher={JMLR. org}
}

@inproceedings{xu2023fsl,
  title={FSL-HD: Accelerating Few-Shot Learning on ReRAM using Hyperdimensional Computing},
  author={Xu, Weihong and Kang, Jaeyoung and Rosing, Tajana},
  booktitle={DATE},
  pages={1--6},
  year={2023},
  organization={IEEE}
}

@inproceedings{salamat2019f5,
  title={F5-hd: Fast flexible fpga-based framework for refreshing hyperdimensional computing},
  author={Salamat, Sahand and Imani, Mohsen and Khaleghi, Behnam and Rosing, Tajana},
  booktitle={Proceedings of the 2019 ACM/SIGDA International Symposium on Field-Programmable Gate Arrays},
  pages={53--62},
  year={2019}
}

@article{kang2022openhd,
  title={Openhd: A gpu-powered framework for hyperdimensional computing},
  author={Kang, Jaeyoung and Khaleghi, Behnam and Rosing, Tajana and Kim, Yeseong},
  journal={IEEE Transactions on Computers},
  volume={71},
  number={11},
  pages={},
  year={2022},
  publisher={IEEE}
}

@article{zhang2023accelerating,
  title={: Accelerating Event-based Workloads with HyperDimensional Computing and Spiking Neural Networks},
  author={Zhang, Tianqi and Morris, Justing and Stewart, Kenneth and Lui, Hin Wai and Khaleghi, Behnam and Thomas, Anthony and Goncalves-Marback, Thiago and Aksanli, Baris and Neftci, Emre O and Rosing, Tajana},
  journal={TCAD},
  year={2023},
  publisher={IEEE}
}

@article{williams2000using,
  title={Using the Nystr{\"o}m method to speed up kernel machines},
  author={Williams, Christopher and Seeger, Matthias},
  journal={NeurIPS},
  volume={13},
  year={2000}
}

@inproceedings{joshi2017language,
  title={Language geometry using random indexing},
  author={Joshi, Aditya and Halseth, Johan T and Kanerva, Pentti},
  booktitle={Quantum Interaction: 10th International Conference, QI 2016, San Francisco, CA, USA, July 20-22, 2016, Revised Selected Papers 10},
  pages={265--274},
  year={2017},
  organization={Springer}
}

@inproceedings{miranda2022hyperdimensional,
  title={Hyperdimensional computing encoding schemes for improved image classification},
  author={Miranda, Victor and d'Aliberti, Olivia},
  booktitle={HST},
  pages={1--9},
  year={2022},
  organization={IEEE}
}

@inproceedings{imani2020dual,
  title={Dual: Acceleration of clustering algorithms using digital-based processing in-memory},
  author={Imani, Mohsen and Pampana, Saikishan and Gupta, Saransh and Zhou, Minxuan and Kim, Yeseong and Rosing, Tajana},
  booktitle={MICRO},
  pages={356--371},
  year={2020},
  organization={IEEE}
}

@inproceedings{imani2019bric,
  title={Bric: Locality-based encoding for energy-efficient brain-inspired hyperdimensional computing},
  author={Imani, Mohsen and Morris, Justin and Messerly, John and Shu, Helen and Deng, Yaobang and Rosing, Tajana},
  booktitle={DAC},
  pages={1--6},
  year={2019}
}

@article{rosenblatt1958perceptron,
  title={The perceptron: a probabilistic model for information storage and organization in the brain.},
  author={Rosenblatt, Frank},
  journal={Psychological review},
  volume={65},
  number={6},
  pages={386},
  year={1958},
  publisher={American Psychological Association}
}

@inproceedings{
yu2022understanding,
title={Understanding Hyperdimensional Computing for Parallel Single-Pass Learning},
author={Tao Yu and Yichi Zhang and Zhiru Zhang and Christopher De Sa},
booktitle={NeurIPS},
editor={Alice H. Oh and Alekh Agarwal and Danielle Belgrave and Kyunghyun Cho},
year={2022},
url={https://openreview.net/forum?id=8ON84BdnSn}
}

@article{kleyko2023survey,
  title={A survey on hyperdimensional computing aka vector symbolic architectures, part ii: Applications, cognitive models, and challenges},
  author={Kleyko, Denis and Rachkovskij, Dmitri and Osipov, Evgeny and Rahimi, Abbas},
  journal={ACM Computing Surveys},
  volume={55},
  number={9},
  pages={1--52},
  year={2023},
  publisher={ACM New York, NY}
}

@article{shimodaira2001dynamic,
  title={Dynamic time-alignment kernel in support vector machine},
  author={Shimodaira, Hiroshi and Noma, Ken-ichi and Nakai, Mitsuru and Sagayama, Shigeki},
  journal={NeurIPS},
  volume={14},
  year={2001}
}

@article{harris1954distributional,
  title={Distributional structure},
  author={Harris, Zellig S},
  journal={Word},
  volume={10},
  number={2-3},
  pages={146--162},
  year={1954},
  publisher={Taylor \& Francis}
}

@article{chen2019powerful,
  title={Are powerful graph neural nets necessary? a dissection on graph classification},
  author={Chen, Ting and Bian, Song and Sun, Yizhou},
  journal={arXiv preprint arXiv:1905.04579},
  year={2019}
}

@inproceedings{charikar2002similarity,
  title={Similarity estimation techniques from rounding algorithms},
  author={Charikar, Moses S},
  booktitle={Proceedings of the thiry-fourth annual ACM symposium on Theory of computing},
  pages={380--388},
  year={2002}
}

@article{musco2017recursive,
  title={Recursive sampling for the nystrom method},
  author={Musco, Cameron and Musco, Christopher},
  journal={NeurIPS},
  volume={30},
  year={2017}
}

@misc{misc_molecular_biology,
  author       = {Harley, C. and Reynolds, R. and Noordewier, M.},
  title        = {{Molecular Biology (Promoter Gene Sequences)}},
  year         = {1990},
  howpublished = {UCI Machine Learning Repository},
  note         = {{DOI}: https://doi.org/10.24432/C5S01D}
}

@article{dobson2003distinguishing,
  title={Distinguishing enzyme structures from non-enzymes without alignments},
  author={Dobson, Paul D and Doig, Andrew J},
  journal={Journal of molecular biology},
  volume={330},
  number={4},
  pages={771--783},
  year={2003},
  publisher={Elsevier}
}

@article{sutherland2003spline,
  title={Spline-fitting with a genetic algorithm: A method for developing classification structure- activity relationships},
  author={Sutherland, Jeffrey J and O'brien, Lee A and Weaver, Donald F},
  journal={Journal of chemical information and computer sciences},
  volume={43},
  number={6},
  pages={1906--1915},
  year={2003},
  publisher={ACS Publications}
}

@article{debnath1991structure,
  title={Structure-activity relationship of mutagenic aromatic and heteroaromatic nitro compounds. correlation with molecular orbital energies and hydrophobicity},
  author={Debnath, Asim Kumar and Lopez de Compadre, Rosa L and Debnath, Gargi and Shusterman, Alan J and Hansch, Corwin},
  journal={Journal of medicinal chemistry},
  volume={34},
  number={2},
  pages={786--797},
  year={1991},
  publisher={ACS Publications}
}

@inproceedings{riesen2008iam,
  title={IAM graph database repository for graph based pattern recognition and machine learning},
  author={Riesen, Kaspar and Bunke, Horst},
  booktitle={SSPR \& SPR 2008, Orlando, USA, December 4-6, 2008. Proceedings},
  pages={287--297},
  year={2008},
  organization={Springer}
}

@inproceedings{zhang2018end,
  title={An end-to-end deep learning architecture for graph classification},
  author={Zhang, Muhan and Cui, Zhicheng and Neumann, Marion and Chen, Yixin},
  booktitle={AAAI},
  volume={32},
  number={1},
  year={2018}
}

@article{xu2018powerful,
  title={How powerful are graph neural networks?},
  author={Xu, Keyulu and Hu, Weihua and Leskovec, Jure and Jegelka, Stefanie},
  journal={arXiv preprint arXiv:1810.00826},
  year={2018}
}

@article{amouzad2024graph,
  title={Graph isomorphism U-Net},
  author={Amouzad, Alireza and Dehghanian, Zahra and Saravani, Saeed and Amirmazlaghani, Maryam and Roshanfekr, Behnam},
  journal={Expert Systems with Applications},
  volume={236},
  pages={121280},
  year={2024},
  publisher={Elsevier}
}

@article{leslie2004fast,
  title={Fast string kernels using inexact matching for protein sequences.},
  author={Leslie, Christina and Kuang, Rui and Bennett, Kristin},
  journal={JMLR},
  volume={5},
  number={9},
  year={2004}
}

@article{rives2021biological,
  title={Biological structure and function emerge from scaling unsupervised learning to 250 million protein sequences},
  author={Rives, Alexander and Meier, Joshua and Sercu, Tom and Goyal, Siddharth and Lin, Zeming and Liu, Jason and Guo, Demi and Ott, Myle and Zitnick, C Lawrence and Ma, Jerry and others},
  journal={Proceedings of the National Academy of Sciences},
  volume={118},
  number={15},
  pages={e2016239118},
  year={2021},
  publisher={National Acad Sciences}
}

@article{devlin2018bert,
  title={Bert: Pre-training of deep bidirectional transformers for language understanding},
  author={Devlin, Jacob and Chang, Ming-Wei and Lee, Kenton and Toutanova, Kristina},
  journal={arXiv preprint arXiv:1810.04805},
  year={2018}
}

@article{qiu2020pre,
  title={Pre-trained models for natural language processing: A survey},
  author={Qiu, Xipeng and Sun, Tianxiang and Xu, Yige and Shao, Yunfan and Dai, Ning and Huang, Xuanjing},
  journal={Science China technological sciences},
  volume={63},
  number={10},
  pages={1872--1897},
  year={2020},
  publisher={Springer}
}

@article{raffel2020exploring,
  title={Exploring the limits of transfer learning with a unified text-to-text transformer},
  author={Raffel, Colin and Shazeer, Noam and Roberts, Adam and Lee, Katherine and Narang, Sharan and Matena, Michael and Zhou, Yanqi and Li, Wei and Liu, Peter J},
  journal={JMLR},
  volume={21},
  number={140},
  pages={1--67},
  year={2020}
}

@article{morris2020tudataset,
  title={Tudataset: A collection of benchmark datasets for learning with graphs},
  author={Morris, Christopher and Kriege, Nils M and Bause, Franka and Kersting, Kristian and Mutzel, Petra and Neumann, Marion},
  journal={arXiv preprint arXiv:2007.08663},
  year={2020}
}

@article{voelker2020short,
  title={A short letter on the dot product between rotated fourier transforms},
  author={Voelker, Aaron R},
  journal={arXiv preprint arXiv:2007.13462},
  year={2020}
}

@inproceedings{ah2010normalized,
  title={Normalized kernels as similarity indices},
  author={Ah-Pine, Julien},
  booktitle={PAKDD 2010, Hyderabad, India, June 21-24, 2010. Proceedings. Part II 14},
  pages={362--373},
  year={2010},
  organization={Springer}
}

@misc{towell1992primate,
  title={Primate splice-junction gene sequences (DNA) with associated imperfect domain theory},
  author={Towell, GG and Noordewier, MO and Shavlik, JW},
  year={1992}
}

\end{document}